\documentclass[11pt, a4paper]{article}
\usepackage[a4paper]{geometry}

\usepackage[english]{babel}

\usepackage[utf8]{inputenc}
\usepackage{xspace}
\usepackage{amsmath,amsthm,amssymb,mathtools}
\usepackage{lmodern}
\usepackage{bbm}
\usepackage{url}
\usepackage[caption=false]{subfig}
\usepackage{natbib}
\usepackage{algorithm}
\usepackage[noend]{algpseudocode}
\usepackage{xcolor}
\usepackage{tikz}
\usepackage{graphicx}
\renewcommand{\labelenumi}{(\alph{enumi})}
\renewcommand\theenumi\labelenumi
\usepackage{diagbox}
\usepackage{tabularx}
\newcolumntype{Y}{>{\centering\arraybackslash}X}

\newcommand{\sdooea}{SD-(1+1)~EA\xspace}
\newcommand{\oneoneea}{\ooea}
\newcommand{\rls}{RLS\xspace}
\newcommand{\sdrls}{SD-RLS\xspace}
\newcommand{\sdrlss}{SD-RLS$^*$\xspace}

\newcommand{\oofea}{(1+1)~FEA$_\beta$\xspace}
\newcommand{\ooea}{(1+1)~EA\xspace}

\newcommand{\om}{\textsc{OneMax}\xspace}

\newcommand{\onemax}{\om}
\newcommand{\jump}{\textsc{Jump}\xspace}

\newcommand{\needhighmut}{\textsc{NeedHighMut}\xspace}
\newcommand{\needglobalmut}{\textsc{NeedGlobalMut}\xspace}
\newcommand{\pre}{\textsc{pre}}
\newcommand{\suff}{\textsc{suff}}

\newcommand{\R}{\ensuremath{\mathbb{R}}}
\newcommand{\N}{\ensuremath{\mathbb{N}}}

\DeclareMathOperator{\Prob}{Pr}

\newcommand{\ones}[1]{\lvert #1\rvert_1}

\newcommand{\ie}{i.\,e.\xspace}
\newcommand{\eg}{e.\,g.\xspace}
\newcommand{\wlo}{w.\,l.\,o.\,g.\xspace}

\newcommand{\card}[1]{\lvert #1\rvert}

\DeclareMathOperator{\gap}{gap}
\DeclareMathOperator{\opt}{opt}
\DeclareMathOperator{\im}{Im}

\newcommand{\xopt}{x_{\opt}}
\newcommand{\prob}[1]{\mathord{\Prob}\mathord{\left(#1\right)}}

\newcommand{\expect}[1]{\mathrm{E}\left(#1\right)}

\newcommand{\TG}{TG\xspace}
\newcommand{\erdosrenyi}{Erdős–Rényi\xspace}
\renewenvironment{proof}
{\begin{trivlist}\item\textbf{Proof.}}
{\hspace*{\fill}$\Box$\end{trivlist}}

\newenvironment{proofof}[1]
{\begin{trivlist}\item\textbf{Proof of #1.}}
{\hspace*{\fill}$\Box$\end{trivlist}}

\algnewcommand{\IfThenElse}[3]{
  \State \algorithmicif\ #1\ \algorithmicthen\ #2\ \algorithmicelse\ #3}

\DeclarePairedDelimiter\floor{\lfloor}{\rfloor}

\newtheorem{lemma}{Lemma}
\newtheorem{theorem}{Theorem}

\title{Stagnation Detection with \\Randomized Local Search}
\author{
  Amirhossein Rajabi\\
  Technical University of Denmark \\
	Kgs. Lyngby \\
	Denmark \\
amraj@dtu.dk \\
  \and
    Carsten Witt\\
  Technical University of Denmark \\
	Kgs. Lyngby \\
	Denmark \\
cawi@dtu.dk \\
}

\allowdisplaybreaks[4]
\begin{document}

\maketitle 
\begin{abstract}
Recently a mechanism called stagnation detection was proposed that
automatically adjusts the mutation rate of evolutionary algorithms when 
they encounter local optima. The so-called \sdooea introduced by
Rajabi and Witt (GECCO~2020) adds stagnation detection to the classical 
\ooea with standard bit mutation, which flips each bit independently 
with some mutation rate, and raises the mutation rate when the 
algorithm is likely to have encountered local optima.

In this paper, 
we investigate stagnation detection in the context of the $k$-bit flip
operator of randomized local search that flips $k$ bits chosen uniformly
at random and let stagnation detection adjust the parameter~$k$. We obtain 
improved runtime results compared to the \sdooea amounting to a speed-up of up 
to~$e=2.71\dots$ Moreover, we propose additional schemes that prevent infinite 
optimization times even if the algorithm misses a working choice of~$k$ due 
to unlucky events. Finally, we present an example where standard bit mutation 
still outperforms the local $k$-bit flip with stagnation detection.
\end{abstract} 

\section{Introduction}
Evolutionary Algorithms (EAs) are parameterized algorithms, so it has been ongoing research to discover how to choose their parameters best. Static parameter settings are not efficient for a wide range of problems. Also, given a specific problem, there might be different scenarios during the optimization, which results in inefficiency of one static parameter configuration for the whole run. Self-adjusting mechanisms address this issue as a non-static parameter control framework that can learn acceptable or even near-optimal parameter settings on the fly. See also the survey article \cite{DoerrDoerrParameterBookChapter} for a detailed coverage of static and non-static parameter control.

Many studies have been conducted on frameworks which adjust the mutation rate of different mutation operators, in particular in the standard bit mutation for the search space of bit strings $\{0,1\}^n$ to make the rate efficient on unimodal functions. For example, the $(1+(\lambda,\lambda)$)~GA using the $1/5$-rule can adjust its mutation strength (and also its crossover rate) on \om \citep{DoerrDoerrAlgorithmica18}, resulting in asymptotic 
speed-ups compared to static settings. Likewise, the self-adjusting mechanism in the (1+$\lambda$)~EA with two rates proposed in \cite{DoerrGWYAlgorithmica19} performs on unimodal functions as 
efficiently as the best $\lambda$-parallel unary unbiased black-box algorithm.

The self-adjusting frameworks mentioned above are mainly designed to optimize unimodal functions. Generally, they are not able to suggest an efficient parameter setting where algorithms get stuck in a local optimum since they mainly work based on the number of successes, so there is no signal in such a situation. On multimodal functions, where some specific numbers of bits have to flip to make progress, \textit{Stagnation Detection} (SD) introduced in \cite{RajabiWittGECCO20} can overcome local optima in  efficient time. This module can be added to most of the existing algorithms to leave local optima without any significant increase of the optimization time of unimodal (sub)problems. To our knowledge, no study has put forward other runtime analyses of self-adjusting mechanisms on multimodal functions. However, in a broader context of mutation-based randomized search heuristics, the heavy-tailed mutation presented in \cite{DoerrLMNGECCO17} has been able to leave a local optimum in a much more efficient time than the standard bit mutation does. Moreover, in the context of 
artificial immune systems \citep{CorusOYPPSN18} and
hyperheuristics \citep{LissovoiOWAAAI19},
there are
proofs that specific search operators and selection of low-level 
heuristics 
can speed up multimodal optimization compared to the classical 
mutation operators.

Recent theoretical research on evolutionary algorithms in discrete search spaces mainly considers global mutations which can create all possible points in one iteration. These mutations have been functional in optimization scenarios where information about the difficulties of the local optima is not available. For example, the standard bit mutation which flips each bit independently with a non-zero probability can produce any point in the search space. However, local mutations can only create a fixed set of offspring points. The 1-bit flip mutation that often can be found in the Randomized Local Search algorithm (RLS) can only reach a limited number of search points, which results in being stuck in a local optimum with the elitist selection.
Nevertheless, local mutations may outperform global mutations on unimodal functions and multimodal functions with known gap sizes. It is of special interest to use advantages of local mutations on unimodal (sub)functions additionally to overcome local optima efficiently. 

This paper investigates $k$-bit flip mutation as a local mutation in the context of the above-mentioned stagnation detection mechanism. This mechanism detects when the algorithm is stuck in a local optimum and gradually increases mutation strength (\ie, the 
number of flipped bits)  to a value the algorithm needs to leave the local optimum. 
Similarly, we aim to show that the algorithms using $k$-bit flip can use stagnation detection to tune the parameter~$k$. One of the key benefits of such algorithms is using the efficiency of RLS, which performs very well on unimodal (sub)problems without fear of infinite running time in local optima. An additional advantage of using $k$-bit flip mutation accompanied by stagnation detection is that it overcomes local optima more efficiently than global mutations. 
Moreover, the outcome points out the advantages and practicability of our self-adjusting approach that makes local-mutation algorithms able to optimize functions that have been intractable to solve so far.

We propose two algorithms combining stagnation detection with local mutations. The first algorithm called \sdrls 
gradually increases the mutation strength  when the current strength 
has been unsuccessful in finding improvements for a 
significantly long time. In the most extreme case, the 
strength ends at~$n$, \ie, mutations flipping all bits. 
With high probability, \sdrls has a runtime that is 
by a factor of $(\frac{ne}{m})^m/\binom{n}{m}$ (up to lower-order terms) smaller on functions with 
Hamming gaps of size~$m$  than the \sdooea 
previously considered in \cite{RajabiWittGECCO20}. This improvement is especially strong for small~$m$ and amounts to a factor of $e$ on unimodal functions. 
Although it is unlikely that the algorithm fails to find an 
improvement 
when the current strength allows this, there is a risk that this algorithm misses the ``right'' strength and 
therefore it can have infinite expected runtime. To address 
this, we propose a second algorithm called \sdrlss that 
repeatedly loops over all smaller strengths than the 
last attempted one when it fails to find an improvement. 
This results in expected finite optimization time on all 
problems and only increases the typical 
runtime by lower-order 
terms compared to \sdrls. We also observe that the algorithms 
we obtain can still follow
the same search trajectory as the classical RLS when one-bit flips are sufficient 
to make improvements. In those cases, well-established techniques for the 
analysis of RLS like the fitness-level method carry over to our variant 
enhanced with stagnation detection. This is not necessarily the case in 
related approaches like variable neighborhood
search \citep{HansenMladenovic18} 
and quasirandom evolutionary algorithms \citep{DoerrFW10} both of which employ 
more determinism and do not generally follow the trajectory of RLS.

We shall investigate the two suggested 
algorithms on unimodal functions and 
functions with local optima of different so-called 
gap sizes, corresponding 
to the number of bits that need to be flipped to escape from the 
optima. Many results are obtained following the analysis of the
SD-\oneoneea \citep{RajabiWittGECCO20} which uses a global operator
with self-adjusted mutation strength. In fact, often the general 
proof structure could be taken over almost literally but with 
improved overall bounds. In conclusion,
the self-adjusting local mutation seems to be the
preferred alternative 
to the SD-\oneoneea with global mutation. However, we will also 
investigate carefully chosen scenarios where 
global mutations are superior.

This paper is structured as follows: in
Section~\ref{sec:preliminaries}, we state the classical 
RLS algorithm and introduce our self-adjusting variants with
stagnation detection; 
moreover, we collect important mathematical tools. Section~\ref{sec:sdrls} shows runtime results for the simpler variant \sdrls, concentrating on the probability 
of leaving local optima, while Section~\ref{sec:sdrlss} 
gives a more detailed analysis of the variant \sdrlss on 
benchmark functions like \onemax and \jump. Section~\ref{sec:needglobalmut} analyzes an example function
which the standard \oneoneea with standard bit mutation 
can solve in polynomial time with high probability whereas the $k$-bit flip mutation with stagnation detection 
needs exponential time. Through improved upper bounds, 
we give in Section~\ref{sec:mst} indications for that our approach may also be superior to static settings on instances of the minimum 
spanning tree problem. This problem and other scenarios 
are investigated experimentally in Section~\ref{sec:experiments} before we finally conclude 
the paper.

\section{Preliminaries}
\label{sec:preliminaries}

\subsection{Algorithms}
In this paper, we consider pseudo-boolean functions $f\colon\{0,1\}^n\to \R$ that 
\wlo are to be maximized. One of the first randomized search heuristics studied in the literature is \textit{randomized local search} (\rls) \citep{DoerrDoerrAlgo16}  displayed in Algorithm~\ref{alg:rls-classic}. This heuristic starts with a random search point and then repeats mutating the point by flipping $s$ uniformly chosen bits (without replacement) and replacing it with the offspring if it is not worse than the parent.

\begin{algorithm}[htb]
	\caption{\rls with static strength~$s$}
	\label{alg:rls-classic}
	\begin{algorithmic}
		\State Select $x$ uniformly at random from $\{0, 1\}^n$
		\For{$t \gets 1, 2, \dots$}
		\State Create $y$ by flipping $s$ bit(s) in a copy of $x$.
		\If{$f(y) \ge f(x)$}
		\State $x \gets y$.
		\EndIf
		\EndFor
	\end{algorithmic}
\end{algorithm}

The \emph{runtime} or the \emph{optimization time} of a heuristic on a function~$f$ 
 is the first point time~$t$ where a search point of maximal fitness has been created; often  the expected runtime, \ie, the expected value of this time, is analyzed.
 
 Theoretical research on evolutionary algorithms mainly studies algorithms on simple unimodal well-known benchmark problems 
 like \[
 \onemax(x_1,\dots,x_n)\coloneqq \ones{x},\]
 but also on the multimodal $\jump_m$ 
 function with gap size~$m$ defined as follows:
 \[
 \jump_m(x_1,\dots,x_n) = 
 \begin{cases}
 m + \ones{x} & \text{ if $\ones{x}\le n-m$ or $\ones{x}=n$}\\
 n-\ones{x} & \text{ otherwise}
 \end{cases}
 \]

  The mutation used in \rls is a local mutation as it only produces
 a limited number of offspring.  This mutation, which we call $s$-flip in the following (in the introduction, we used the classical name $k$-bit flip), flips exactly $s$ bits randomly chosen from the bit string of length~$n$, so for any point $x\in \{0,1\}^n$, \rls can just sample from $\binom{n}{s}$ possible points. As a result, $s$-flip
 is often more efficient compared to global mutations when we know the difficulty of making progress 
  since the algorithm just looks at a certain part of the search space. To be more precise, we recall the so-called \emph{gap} of the point $x\in \{0,1\}^n$ defined in \cite{RajabiWittGECCO20}
as the minimum Hamming distance to points with the strictly larger fitness function value. Formally,
\[\gap(x)\coloneqq \min\{H(x,y) : f(y)>f(x) , y\in \{0,1\}^n\}. \]
It is not possible to make progress by flipping less than $\gap(x)$ bits of the current search point~$x$. However, if the algorithm uses the $s$-flip with $s=\gap(x)$, it can make progress with a positive probability. In addition, on unimodal functions where the gap of all points in the search space (except for global optima) is one, the algorithm makes progress with strength~$s=1$. 

Nevertheless, understanding the difficulty of a local optimum has not generally been possible so far, and benefiting from domain knowledge to use it to determine the strength is not always feasible in the perspective of black-box optimization. Therefore, despite the advantages of $s$-flip, global mutations, \eg standard bit mutation, which can produce any point in the search space, have been used in the literature frequently. For example, the \ooea that uses a similar approach to Algorithm~\ref{alg:rls-classic} benefits from standard bit mutation that implicitly uses the binomial distribution to determine how many bits must flip. Consequently, even if the algorithm uses strength~$1$ (e.g. mutation rate~$1/n$), with a positive probability, the algorithm can escape from any local optimum.

We study the search and success probability 
of Algorithm~\ref{alg:rls-classic} and its relation to stagnation 
detection more closely. 
With similar arguments as presented in \cite{RajabiWittGECCO20}, if the gap of the current search point is~$1$ then the algorithm makes an improvement with probability~$1/R$ at strength~$1$, 
and the probability of not finding it in $n\ln R$ steps is at most 
$
(1-1/n)^{n\ln R}\le 1/R$  (where $R$ 
is a parameter to be discussed).
Similarly, the probability of not finding an improvement for a point with gap of $k$ within $\binom{n}{k} \ln R$ steps is at most 
\[
\left(1- \frac 1{\binom{n}{k}}\right)^{\binom{n}{k}\ln R}\le \frac 1R.
\]
Hence, after $\binom{n}{k} \ln R$ steps without improvement there is a probability of at least $1-1/R$ that no improvement at Hamming distance~$k$ exists, so for enough large~$R$ the probability of failing is small.

We consider this idea to develop the first algorithm. We add the stagnation detection mechanism to \rls to manage the strength~$s$. As shown in Algorithm~\ref{alg:sdrls},
hereinafter called \sdrls, the initial strength is~1. Also, there is a counter~$u$ for counting the number of unsuccessful steps to find the next  after the last success. When the counter exceeds the threshold of $\binom{n}{s}\ln R$, strength~$s$ is increased by one, and when the algorithm makes progress, the counter and strength are reset to their initial values. In the case that the algorithm is failed to have a success where the strength is equal to the gap of the current search point, the algorithm misses the chance of making progress. Therefore, with probability $1/R$, the optimization time would be infinitive. Choosing a large enough~$R$ to have an overwhelming large probability of making progress could be a solution to this problem. However, we propose another algorithm that resolves this issue, although the running time is not always as efficient as with Algorithm~\ref{alg:sdrls}.

In Algorithm~\ref{alg:sdrls_star}, hereinafter called \sdrlss,  we introduce a new variable~$r$ called radius. This parameter determines the largest Hamming distance from the current search point that algorithm must investigate. In details, when the radius becomes $r$, the algorithm starts with strength~$r$ (\ie, $s=r$) and when the threshold is exceeded, it decreases the strength by one as long as the strength is greater than 1. This results in a more robust behavior. In the case that the threshold exceeds and the current strength is $1$, the radius is increased by one to cover a more expanded space. Also, when the radius exceeds $n/2$, the algorithm increases the radius to $n$, which means that the algorithm covers all possible strengths between~$1$ and~$n$. We note that 
the strategy of repeatedly 
returning to lower strengths remotely resembles the 
$1/5$-rule with rollbacks proposed in \cite{BassinBuzdalovGECCO19}.

\begin{algorithm}[htb]
	\caption{RLS with stagnation detection (\sdrls)}
	\label{alg:sdrls}
	\begin{algorithmic}
		\State Select $x$ uniformly at random from $\{0, 1\}^n$ and set $s_1 \gets 1$.
		\State $u\gets 0$.
		\For{$t \gets 1, 2, \dots$}
		\State Create $y$ by flipping $s_t$ bits in a copy of $x$ uniformly.
		\State $u\gets u+1$.
		\If{$f(y) > f(x)$}
		\State $x \gets y$.
		\State $s_{t+1}\gets 1$.
		\State $u\gets 0$.
		\ElsIf {$f(y) = f(x)$ \textbf{and} $s_t=1$}
		\State $x \gets y$.

		\EndIf
		\If{$u > \binom{n}{s}\ln R$} 
		\State $s_{t+1}\gets \min\{s_t+1,n\}$.
		\State $u\gets 0$.
		\Else
		\State $s_{t+1}\gets s_t$.
		\EndIf
		\EndFor
	\end{algorithmic}
\end{algorithm}

\begin{algorithm}[htb]
	\caption{RLS with robust stagnation detection (\sdrlss)}
	\label{alg:sdrls_star}
	\begin{algorithmic}
		\State Select $x$ uniformly at random from $\{0, 1\}^n$ and set $r_1 \gets 1$ and $s_1 \gets 1$.
		\State $u\gets 0$.
		\For{$t \gets 1, 2, \dots$}
		\State Create $y$ by flipping $s_t$ bits in a copy of $x$ uniformly.
		\State $u\gets u+1$.
		\If{$f(y) > f(x)$}
		\State $x \gets y$.
		\State $s_{t+1}\gets 1$.
		\State $r_{t+1}\gets 1$.
		\State $u\gets 0$.
		\ElsIf {$f(y) = f(x)$ \textbf{and} $r_t=1$}
		\State $x \gets y$.

		\EndIf
		\If{$u > \binom{n}{s_t}\ln R$}
		\If{$s_t=1$}
		    \IfThenElse {$r_t<n/2$}{$r_{t+1}\gets r_t+1$}{$r_{t+1}\gets n$}
		    \State $s_{t+1}\gets r_{t+1}.$
		    \Else
		    \State $r_{t+1}\gets r_t$.
		    \State $s_{t+1}\gets s_t-1$.
		\EndIf
		\State $u\gets 0$.
		\Else
		\State $s_{t+1}\gets s_t$.
	    \State $r_{t+1}\gets r_t$.
		\EndIf
		\EndFor
	\end{algorithmic}
\end{algorithm}

The parameter~$R$ represents the probability of failing to find an improvement at the ``right'' strength. More precisely, as we will see in Theorem~\ref{theo:gapxforrls} and Lemma~\ref{lem:failure-probability} (for \sdrls and \sdrlss, respectively), the probability of not finding an improvement where there is a potential of making progress is at most $1/R$.
We recommend $R\ge \card{\im f}$ for \sdrls (where $\im f$ is the 
image set of~$f$), and for a constant $\epsilon$, $R\ge n^{3+\epsilon}\cdot \card{\im f}$ for \sdrlss, resulting in that the probability of ever missing an improvement at the right strength is sufficiently small throughout the run.

\subsection{Mathematical tools}
The following lemma containing some combinatorial inequalities will be used in the analyses of the algorithms. The first part of the lemma seems to be well known and has already been proved in \cite{MathoverflowML17} and is also a consequence of Lemma~1.10.38 
in \cite{DoerrProbabilisticBookChapter}.  The second part follows from elementary manipulations.

\begin{lemma} \label{lem:partial-sum}
\parbox[t]{\textwidth}{For any integer $m\le n/2$, we have 
	\begin{enumerate}
	    \item $\sum_{ i=1}^{m}\binom{n}{i}\leq \frac{n-(m-1)}{n-(2m-1)} \binom{n}{m}$,
        \item $\binom{n}{M}\le \binom{n}{m}\left(\frac{n-m}{m}\right)^{M-m}$ for $m<M<n/2$. \label{lem:partial-sum:M}
	\end{enumerate}}
\end{lemma}
\begin{proof}
    \begin{enumerate}
        \item We use the following proof due to \cite{MathoverflowML17}. Through the equation $\binom{n}{k-1}=\frac{k}{n-k+1}\binom{n}{k}$ which comes from the definition of the binomial coefficient and geometric series sum formula, we achieve the following result:
	\begin{align*}
		\frac{\sum_{ i=1}^{m}\binom{n}{i}}{\binom{n}{m}} 
		& = \frac{\binom{n}{m}}{\binom{n}{m}}+\frac{\binom{n}{m-1}}{\binom{n}{m}}+\dots+\frac{\binom{n}{1}}{\binom{n}{m}}\\
		& < 1+\frac{m}{n-m+1}+\frac{m(m-1)}{(n-m+1)(n-m+2)}+\dots \\
		& \le 1+\frac{m}{n-m+1}+\left(\frac{m}{n-m+1}\right)^2+\dots
		 =  \frac{n-(m-1)}{n-(2m-1)}
	\end{align*}	
        \item For $t\ge m$, by using $\binom{n}{m}=\frac{n-m+1}{m}\binom{n}{m-1}$, we have \begin{align*}
            \binom{n}{M} &=\frac{(n-M+1)\dots(n-M+t)}{M\dots(M-t+1)}\binom{n}{M-t} \\
            & \le \left(\frac{n-M+t}{M-t}\right)^t\binom{n}{M-t}
        \end{align*}
    Thus, by setting $m=M-t$, we obtain the statement.
    \end{enumerate}
\end{proof}

\section{Analysis of the Algorithm \sdrls}
\label{sec:sdrls}
In this section, we study the first algorithm called \sdrls, see Algorithm~\ref{alg:sdrls}. In the beginning of the section, we show upper and lower bounds on the time for escaping from local optima.
Then, in Theorem~\ref{theo:unimodalrls}, we show the important result that on unimodal functions, 
 \sdrls with probability~$1-\card{\im f}/R$
behaves in the same way as \rls with strength~$1$, including the same asymptotic bound on the expected optimization time.

The following theorem shows the time \sdrls takes with probability $1-1/R$ to make progress of search point~$x$ with a gap of $m$. 
\begin{theorem} \label{theo:gapxforrls}
Let $x\in\{0,1\}^n$ be the current search point of \sdrls
 on a  pseudo-boolean function $f\colon\{0,1\}^n\to \R$.
	Define $T_x$ as the time to create a strict improvement if $\gap(x)=m$. Let $U$ be the event of finding an improvement at Hamming distance~$m$. Then, we have
		\begin{align*} 
		\expect{T_x\mid U} \le     
		\begin{cases}
		\binom{n}{m}(1+O(\frac{m\ln R}{n})) & \text{if } m=o(n),\\
		O\left(\binom{n}{m} \ln{R} \right) & \text{if } m=\Theta(n) \wedge m<n/2, \\
		O(2^n\ln R) & \text{if } m\ge n/2.
		\end{cases}
	\end{align*}
	 Moreover, $\prob{U}\ge 1-1/R$.
\end{theorem}

Compared to the corresponding theorems in \cite{RajabiWittGECCO20}, the 
bounds in Theorem~\ref{theo:gapxforrls} are by a factor of $(\frac{ne}{m})^m/\binom{n}{m}$ (up to lower-order terms) smaller. This 
 speedup is roughly $e$ for $m=1$, \ie, unimodal functions (like \onemax) but 
becomes less pronounced for larger~$m$ since, intuitively, the number of flipped 
bits in a standard bit mutation will become more and more concentrated  and start resembling the $m$-bit flip mutation.

\begin{proofof}{Theorem~\ref{theo:gapxforrls}}
    The algorithm~\sdrls can make an improvement only where the current strength~$s$ is equal to $m$ and the probability of not finding an improvement during this phase is
    \[
        \left(1-\binom{n}{m}^{-1}\right)^{\binom{n}{m}\ln R} \le \frac 1{R}.
    \]
    
    If the improvement event happens, the running time of the algorithm to escape from this local optimum is
    \begin{align*}
	\expect{T_x \mid U} 
	&< \underbrace{ \sum_{i=1}^{m-1}\binom{n}{i} \ln R }_{=:S_1} +
	\underbrace{ \binom{n}{m}}_{=:S_2},
	\end{align*}
	where $S_1$ is the number of iterations for $s<m$ and $S_2$ is the expected number of iterations needed to make an improvement where $s=m$.
	
	By using Lemma~\ref{lem:partial-sum} for $m<n/2$, we have
	\begin{align*}
	\expect{T_x \mid U} 
	&< \sum_{i=1}^{m-1}\binom{n}{i} \ln R + \binom{n}{m} < \frac{n-m+2}{n-2m+3}\binom{n}{m-1}\ln R + \binom{n}{m}\\
	& = \frac{n-m+2}{n-2m+3} \cdot \frac{m}{n-m+1}\binom{n}{m}\ln R + \binom{n}{m} \\
	& = \binom{n}{m} \left( \frac{n-m+2}{n-2m+3} \cdot \frac{m}{n-m+1}\ln R +1 \right),
	\end{align*}
	and for $m\ge n/2$, we know that $\sum_{i=1}^{n}\binom{n}{i}<2^n$, so we can compute
	\begin{align*}
	\expect{T_x \mid U} 
	&= \sum_{i=1}^{m-1}\binom{n}{i} \ln R + \binom{n}{m} \le O(2^n \ln R)
	\end{align*}
Altogether we achieve
	\begin{align*}
		\expect{T_x\mid U} \le     
		\begin{cases}
		\binom{n}{m}(1+O(\frac{m\ln R}{n})) & \text{if } m=o(n),\\
		O\left(\binom{n}{m} \ln{R} \right) & \text{if } m=\Theta(n) \wedge m<n/2, \\
		O(2^n\ln R) & \text{if } m\ge n/2. \makebox[0pt]{\hspace*{6.9cm} }
		\end{cases}
	\end{align*}	 
\end{proofof}

Using the previous lemma, we obtain the following result that allows us to reuse existing results for RLS on unimodal functions.

\begin{theorem} \label{theo:unimodalrls} 
	Let $f\colon\{0,1\}^n\to \R$ be a unimodal function and consider  \sdrls with $R\ge\card{\im f}$. Then, with probability at least $1-\frac{\card{\im f}}{R}$, the 
	\sdrls never increases the radius and behaves stochastically like \rls before finding an optimum of~$f$. 
\end{theorem}
\begin{proof}
	As on unimodal functions, the gap of all points is~$1$, the probability of not finding and improvement is \[\left(1-\binom{n}{m}^{-1}\right)^{\binom{n}{m}\ln R} \le \frac 1{R}.\]
	This argumentation holds for each improvement that has to be found. 
	Since at most $\card{\im f}$ improving steps happen before 
	finding the optimum, by a union bound the probability of algorithm~\ref{alg:sdrls} ever increasing the strength beyond~$1$ is at most 
	$\card{\im (f)}\frac{1}{R}$, which proves the lemma.
 \end{proof}

With these two general results, we conclude the analysis of \sdrls and turn to the variant \sdrlss that 
always has finite expected optimization time. In fact, 
we will present similar results in general optimization 
scenarios and supplement them by analyses on specific 
benchmark functions. It is possible
to analyze the simpler \sdrls on these benchmark functions 
as well, but we do not feel that this gives additional 
insights.

\section{Analysis of the Algorithm \sdrlss}
\label{sec:sdrlss}

	In this section, we turn to the algorithm \sdrlss that iteratively returns to lower strengths to avoid missing the ``right'' strength. 
We recall $T_x$ as the number of steps \sdrlss takes to find an 
improvement point from the current search point $x$. Let phase~$r$ consists of all points of time where radius~$r$ is used in the algorithm. When the algorithm enters phase~$r$, it starts with strength $r$, but when the counter exceeds the threshold, the strength decreases by one as long as it is greater than 1. In the case of strength 1, the radius~$r$ is increased to $r+1$ (or to $n$ if $r+1$ is at least $n/2$), so the algorithm enters phase~$r+1$ (or phase~$n$).

Let \emph{$E_r$} be the event of \textbf{not} finding the optimum within 
phase~$r$, and \emph{$U_{i}^{j}$} for $j>i$ be the event of not finding the optimum 
during phases $i$ to $j-1$ and finding it in phase $j$. In other words, $U_i^j=E_i\cap \dots \cap E_{j-1} \cap \overline{E_{j}}$. For $i=j$, we define $\emph{$U_{i}^{i}$}=\overline{E_i}$. We obtain the following result on the failure probability which follows from the fact that the algorithm tries to find an improvement for $\binom{n}{m}\ln R$ iterations with a probability of success of $\binom{n}{m}^{-1}$ when the radius is at least $m$.

\begin{lemma} \label{lem:failure-probability}
	Let $x\in\{0,1\}^n$ be the current search point of \sdrlss on a pseudo-boolean fitness function $f\colon\{0,1\}^n\to \R$ 
	and let $m=\gap(x)$. 
	Then 
	\begin{align*}
		\prob{E_r} \le     
		\begin{cases}
		\frac{1}{R} & \text{ if } m \le r<\frac n2\\
		0 & \text{if } r=n.
		\end{cases}
	\end{align*}
\end{lemma}

The following lemma bounds the time to leave a local optimum conditional 
on that the ``right'' strength was missed.

\begin{proof}
Assume $r \ge m$. During phase~$r$, the algorithm spends $\binom{n}{m} \ln R$ steps at strength~$m$ until it changes the strength or phase. Then, the 
		probability of not improving at strength~$s=m$ in phase~r is at most
		\begin{align*}
		\prob{E_r} & = \left(1-\binom{n}{m}^{-1}\right)^{\binom{n}{m}\ln R} \le \frac{1}{R}.
		\end{align*}
		
	During phase~$n$, the algorithm does not change radius~$r$ anymore, and it continues to flip $s$ bits with different $s$ containing $m$ until making progress so the probability of eventually failing to find the improvement in this phase is~$0$. 
\end{proof}
\begin{lemma} \label{lem:runtime-of-Em}
    Let $x\in\{0,1\}^n$ with $m=\gap(x)<n/2$ be the current search point of \sdrlss with $R\ge n^{3+\epsilon}\cdot \card{\im f}$ for an arbitrary constant $\epsilon>0$
 on a  pseudo-boolean function $f\colon\{0,1\}^n\to \R$ and $T_x$ be the time to create a strict improvement. Then, we have
		\begin{align*}
		\expect{T_x\mid E_m} = o\left(\frac R{\card{\im f}}\binom{n}{m}\right),
	\end{align*}
	where $E_m$ is the event of not finding an optimum when the radius~$r$ equals $m$.
\end{lemma}
The reason behind the factor~$R/\card{\im f}$ in Lemma~\ref{lem:runtime-of-Em} is that for proving a running time of \sdrlss on a function like~$f$, the event~$E_m$  happens with probability~$1/R$ for each point in $\im f$, so in the worst case, during the run, there are expected $\card{\im f}/R$ search points where the counter exceeds the threshold, resulting in an expected number of at most  $\card{\im f}/R \cdot o(R/\card{\im f} \binom{n}{m})$ extra iterations for the whole run in the case of exceeding the thresholds. Also, note that we always have $R/\card{\im f} = \Omega(1)$ since according to the assumption, $R>\card{\im f}$.

\begin{proofof}{Lemma \ref{lem:runtime-of-Em}}
	Using the law of total probability with respect to the events $U^i_{m+1}$ defined above, we have
\begin{align*}
	\expect{T_x \mid E_m} = \underbrace{ \sum_{i=m+1}^{\floor{n/2-1}} \expect{T_x \mid U_{m+1}^{i}} \prob{U_{m+1}^{i}}}_{=:S_1} + \underbrace{\expect{T_x \mid U_{m+1}^n} \prob{U_{m+1}^n}}_{=:S_2}.
	\end{align*}
	
  	In order to estimate $\prob{U_{m+1}^i}$, using Lemma~\ref{lem:failure-probability}, we have 
  	\[\prob{U_{m+1}^i}< \prod_{j=m+1}^{i-1}\prob{E_j}<R^{-(i-m-1)}.\]
  	
  	For $S_1$, by using Lemma~\ref{lem:partial-sum} multiple times, we compute
	\begin{align*}
	&\sum_{i=m+1}^{\floor{n/2-1} } \expect{T_x \mid U_{m+1}^{i}} \prob{U_{m+1}^{i}}   = \sum_{i=m+1}^{\floor{n/2-1}} \sum_{r=1}^{i}\sum_{s=1}^{r} \binom{n}{s}\ln R \cdot R^{-(i-m-1)} \\
	\qquad & \le \sum_{i=m+1}^{\floor{n/2-1}} \sum_{r=1}^{i} \frac{n-r+1}{n-2r+1}\binom{n}{r}\ln R  \cdot R^{-(i-m-1)} \\
	& \le R \sum_{i=m+1}^{\floor{n/2-1}}  \left(\frac{n-i+1}{n-2i+1}\right)^2\binom{n}{i}\ln R  \cdot R^{-(i-m)} \\
	& \le R\sum_{i=m+1}^{\floor{n/2-1}}  \left(\frac{n-i+1}{n-2i+1}\right)^2\left(\frac{n-m}{m}\right)^{i-m}\binom{n}{m} \ln R  \cdot R^{-(i-m)} \\
	& \le \frac{R}{\card{\im f}}\binom{n}{m} \sum_{i=m+1}^{\floor{n/2-1}}  \left(\frac{n-i+1}{n-2i+1}\right)^2\left(\frac{n-m}{n\cdot m}\right)^{i-m}  \cdot \frac{n^{i-m}\card{\im f}\ln R}{R^{(i-m)}}. 
	\end{align*}
	Using the fact that $R\ge n^{3+\epsilon}\cdot \card{\im f}$ and $i>m$, the 
	last expression 
	is bounded from above by $o(R/\card{\im f}\binom{n}{m})$.
	
Regarding $S_2$, when radius~$r$ is increased to $n$, the algorithm mutates $s$ bits of the the current search point for all possible strengths of $1$ to $n$ periodically. In each cycle through different strengths, according to lemma \ref{lem:failure-probability}, the algorithm escapes from the local optimum with probability $1-1/R$ so there are  $R/(R-1)$ cycles in expectation via geometric distribution. Besides, each cycle of radius~$n$ costs $\sum_{i=s}^{n}\binom{n}{s}\ln R$. Overall, we have $\frac{R}{R-1}\sum_{s=1}^{n}\binom{n}{s}\ln R$ extra fitness function calls if the algorithm fails to find the optimum in the first $\floor{n/2-1}$ phases happened with the probability of $R^{-(\floor{n/2}-m-1)}$. Thus, we have
\begin{align*}
    &\expect{T_x \mid U_{m+1}^n}\prob{U_{m+1}^n} \le \left( \sum_{r=1}^{\floor{n/2}}\sum_{s=1}^{r}\binom{n}{s}\ln R + \frac{R}{R-1} \sum_{s=1}^{n}\binom{n}{s}\ln R \right) R^{-(\floor{n/2}-m-1)} \\ 
    &\qquad \le R\cdot \left( 1+\frac{R}{R-1}\right) 2\sum_{s=1}^{n/2}\binom{n}{s}\ln R \cdot R^{-(\floor{n/2}-m)} \\
    &\qquad \le R\cdot\left( 1+\frac{R}{R-1}\right) 2\sum_{s=1}^{n/2}\left( \frac{n-m}{m} \right)^{s-m}\binom{n}{m}\ln R \cdot R^{-(\floor{n/2}-m)}\\
    &\qquad = \frac R{\card{\im f}} \cdot \left( 1+\frac{R}{R-1}\right) 2 \binom{n}{m} \sum_{s=1}^{n/2}\left( \frac{n-m}{n \cdot m} \right)^{s-m}\frac{\card{\im f}n^{s-m}\ln R}{R^{(\floor{n/2}-m)}} \\
    &\qquad =o\left(\frac R{\card{\im f}}\binom{n}{m}\right)
\end{align*}
Altogether, we finally have $\expect{T_x\mid E_m}=S_1+S_2=o\left(R/\card{\im f}\binom{n}{m}\right)$ as suggested.  
\end{proofof}

The following theorem and its proof are similar to Theorem~\ref{theo:gapxforrls} but require a more careful analysis to cover the repeated use of smaller strengths. We note that the bounds differ from Theorem~\ref{theo:gapxforrls} only in 
lower-order terms unless $m$ is very big.

\begin{theorem} \label{theo:gapxforrlsstar}
Let $x\in\{0,1\}^n$ be the current search point of \sdrlss with $R\ge n^{3+\epsilon}\cdot \card{\im f}$ for an arbitrary constant $\epsilon>0$
 on a  pseudo-boolean function $f\colon\{0,1\}^n\to \R$.
	Define $T_x$ as the time to create a strict improvement if $\gap(x)=m$. Then, we have
		\begin{align*}
		\expect{T_x} \le     
		\begin{cases}
		\binom{n}{m}\left(1+O\left(\frac{m^2}{n-2m}\ln R\right)\right) & \text{if } m<n/2\\
		2^nn \ln R & \text{if } m\ge n/2
		\end{cases},
	\end{align*}
	and $\expect{T_x}\ge\binom{n}{m}/W$, where $W$ is the number 
	of strictly better search points at Hamming distance~$m$.
\end{theorem}

\begin{proof}
	Using the law of total probability with respect to $E_m$ defined above as the event of not finding the optimum by the end of phase $m$, we have
	\begin{align*}
	\expect{T_x} &= \underbrace{\expect{T_x \mid \overline{E_m}}\prob{\overline{E_m}}}_{=:S_1} +
	\underbrace{ \expect{T_x \mid E_m} \prob{E_m}}_{=:S_2}.
	\end{align*}
	
	Regarding $S_1$, it takes $\sum_{ i=1}^{m-1} \sum_{ j=1}^{i}\binom{n}{j} \ln R$ steps until \sdrlss 
	increases both radius and strength to~$m$. When the mutation strength is $m$, within an expected number of  $\binom{n}{m}$ steps, a better point will be found. 
	
	In regard to $S_2$, where the optimum is not found by the end of phase $m$, there would be at most $O(R/\card{\im f}\binom{n}{m})$ iterations in expectation through Lemma \ref{lem:runtime-of-Em}. This event, \ie, $E_m$, is happened with the probability of $1/R$.
	
	Altogether, for $m<n/2$, using Lemma~\ref{lem:partial-sum}, we have
	\begin{align*}
	\expect{T_x} &= \prob{\overline{E_m}}\expect{T_x \mid \overline{E_m}} +\prob{E_m}\expect{T_x \mid E_m} \\
	&\le \expect{T_x \mid \overline{E_m}} +\prob{E_m}\expect{T_x \mid E_m} \\
	&\le \sum_{ i=1}^{m-1} \sum_{ j=1}^{i}\binom{n}{j} \ln R + \binom{n}{m} + \frac{1}{R} \cdot o\left( \frac R {\card{\im f}} \binom{n}{m}\right) \\ 
	& \le \sum_{ i=1}^{m-1}
	\frac{n-(i-1)}{n-(2i-1)}
	\binom{n}{i} \ln R + \binom{n}{m} +  o\left(\binom{n}{m}\right)\\
	& \le \left(1+\frac{m}{n-2m+1}\right) \sum_{ i=1}^{m-1}
	\binom{n}{i} \ln R + \binom{n}{m}+ o\left(\binom{n}{m}\right) \\
	& \le \left(1+\frac{m}{n-2m+1}\right) \left(1+\frac{m-1}{n-2m+3}\right) 
	\binom{n}{m} \ln R + \binom{n}{m}+ o\left(\binom{n}{m}\right) \\
	& = \binom{n}{m}O\left(1+\frac{m^2}{n-2m}\ln R\right).
	\end{align*}
	
	For $m\ge n/2$, the algorithm is not able to make an improvement for radius~r less than $n/2$. However, as radius~r is increased to $n$, the algorithm mutates m-bits of the the current search point for all possible strengths of 1 to n periodically. Thus, according to lemma \ref{lem:failure-probability}, the algorithm escapes from the local optimum with probability $1-1/R$ so there are  $R/(R-1)$ cycles in expectation through geometric distribution in this phase. Finally, we compute
	\begin{align*}
	\expect{T_x \mid U_m}  &= \sum_{ i=1}^{\floor{n/2-1}} \sum_{ j=1}^{i}\binom{n}{j} \ln R + \frac{R}{R-1} \sum_{ i=1}^{n} \binom{n}{i} \ln R
	 \le O\left(2^nn\ln R\right).
	\end{align*}

Moreover,  
	the expected number of iterations for making an improvement is at least
	$ \binom{n}{m} $ where the current strength~$s$ equals $m$. The algorithm is not able to have a success with other strengths. Therefore, we have $\expect{T_x}\geq \binom{n}{m}$.  
\end{proof}

Similarly to Lemma~\ref{theo:unimodalrls}, we obtain a relation to RLS 
on unimodal functions and can re-use existing upper bounds based on the 
fitness-level method \citep{WegenerMethods}.

\begin{lemma} \label{lem:unimodalrlsstar}
	Let $f\colon\{0,1\}^n\to \R$ be a unimodal function and consider \sdrlss with $R\ge n^{3+\epsilon}\cdot \card{\im f}$ for an arbitrary constant $\epsilon>0$. Then, with probability at least~$1-\frac{\card{\im f}}{R}$, \sdrlss never increases the radius and behaves stochastically like \rls before finding an optimum of~$f$. 
		
	Denote by $T$ the runtime of \sdrlss on~$f$. Let $f_i$ be the $i$-th fitness value of an increasing order of all fitness values in $f$ and $s_i$ be a lower bound on the probability that \rls finds an improvement from search points with fitness value $f_i$, then 
	\[
	\expect{T} \le  \sum_{i=1}^{ \card{\im f} }\frac 1{s_i} + o(n).
	\]
\end{lemma}
\begin{proof}
	As on unimodal functions, the gap of all points is~$1$, the probability of not finding and improvement is $\left(1-\binom{n}{m}^{-1}\right)^{\binom{n}{m}\ln R} \le \frac 1{R}$.
	This argumentation holds for each improvement that has to be found. 
	Since at most $\card{\im f}$ improving steps happen before 
	finding the optimum, by a union bound the probability of \sdrlss ever increasing the strength beyond~$1$ is at most 
	$\card{\im (f)}\frac{1}{R}$, which proves the lemma.
	
	We let the random set~$W$ contain the search points from which \sdrlss does not find an improvement within phase~$1$ (\ie, while $r_t=1$).
	To prove the second claim, we consider all fitness levels $A_1, \dots, A_{\card{\im f}}$ such that $A_i$ contains search points with fitness value $f_i$ and sum up upper bounds on the expected times to leave each of these fitness levels. Under the 
	condition that the strength is not increased before leaving a fitness level, the worst-case time to leave fitness level~$A_i$ is $1/s_i$ similarly to \rls. Hence,
	we bound the expected optimization time of \sdrlss from above 
	by adding the waiting times on all fitness levels for \rls, which is given by $\sum_{i=1}^{ \card{\im f} } 1/s_i$, and the expected times 
	spent to leave the points in~$W$; formally,  
	\[
	\expect{T} \le \sum_{i=1}^{ \card{\im f} }\frac 1{s_i} + \sum_{x\in W}\expect{T_x}. 
	\]

	Each point in~$\im f$ contributes with probability $\prob{E_1}$ to~$W$. Hence, 
	$\expect{\card{W}}\le \card{\im f}\prob{E_1}$. 
	As on unimodal functions, the gap of all points is~1, by 
	Lemma \ref{lem:runtime-of-Em}, we compute
	\begin{align*}
	\sum_{x\in W}\expect{T_x} 
	&\le \card{\im f}\cdot \prob{E_1}\cdot \expect{E_1} \le \card{\im f}\cdot R^{-1}o\left(\frac R{\card{\im f}}\cdot \binom{n}{1}\right)=o(n).
	\end{align*}
	
	Thus, we finally have
	\begin{align*}
	\expect{T} \le \sum_{i=1}^{ \card{\im f} }\frac 1{s_i}+o(n),
	\end{align*}
	as suggested.
 \end{proof}

Finally, we use the results developed so far to prove a bound on the 
\jump function which seems to be the best available for mutation-based 
hillclimbers.

\begin{theorem}
\label{theo:jump}
	Let $n\in \N$. For all $2\le m$, the expected runtime $\expect{T}$ of \sdrlss with $R\ge n^{4+\epsilon}$ for an arbitrary constant $\epsilon>0$ on $\jump_m$ satisfies 
	\begin{align*}
		\expect{T} \le     
		\begin{cases}
		\binom{n}{m}\left(1+O\left(\frac{m^2}{n-2m}\ln n\right)\right) & \text{if } m<n/2,\\
		O(2^nn \ln n) & \text{otherwise.}
		\end{cases}
	\end{align*}
\end{theorem}

\begin{proof}
	Before reaching the plateau consisting of all points of $n-m$ one-bits, \jump is equivalent to \onemax; hence, 
	according to Lemma~\ref{lem:unimodalrlsstar}, the expected running time \sdrls takes to reach the plateau is at most $O(n\ln n)$. Note that this bound was obtained 
	via the fitness level method with $s_i=(n-i)/n$ as 
	minimum probability for leaving the set of search points 
	with $i$ one-bits.
	
	Every plateau point $x$ with $n-m$ one-bits satisfies $\gap(x)=m$ according to the definition of \jump. 
	Thus, 	using Theorem~\ref{theo:gapxforrlsstar}, the algorithm finds the optimum within expected time
	\[
	 \expect{T_x} \le \begin{cases}
		\binom{n}{m}\left(1+O\left(\frac{m^2}{n-2m}\ln n\right)\right) & \text{if } m<n/2\\
		O(2^nn \ln n) & \text{if } m\ge n/2
		\end{cases}.\]
	This 
	 dominates the expected time of the algorithm before the plateau point and results in the running time in the theorem.  
\end{proof}

\section{An Example Where Global Mutations are Necessary}
\label{sec:needglobalmut}
While our $s$-flip mutation along with stagnation 
detection can outperform the \oneoneea on \jump functions, it is
clear that its different search behavior may be disadvantageous 
on other examples. Concretely, we will present a function 
that has a unimodal path to a local optimum with a large Hamming 
distance to the global optimum. \sdrls will with high probability
follow this path and incur exponential optimization time. 
However, the function has a second gradient that requires 
two-bit flips to make progress. The classical \oneoneea will 
be able to follow this gradient and to arrive at the global
optimum
before one-bit flips have reached the end of the path to the 
local optimum.

In a broader context, our function illustrates an advantage of global mutation operators. By a simple swap of local and 
global optimum, it immediately 
turns into the direct opposite, \ie, an example where using global
instead of local mutations is highly detrimental and increases the runtime from 
polynomial to exponential with overwhelming probability. An 
example of such a function was previously presented in 
\cite{DoerrJansenKleinGECCO08}; however, both the underlying 
construction and the proof of exponential runtime for the \oneoneea 
seem much more complicated than our example.

We will in the following 
define the example function called \needglobalmut 
and give proofs 
for the behavior of \sdrls and \oneoneea. In fact, \needglobalmut 
is obtained from the function \needhighmut defined in \cite{RajabiWittGECCO20} to show disadvantages of 
stagnation detection adjusting the rate of a global mutation 
operator. The only change is to adjust the length of the suffix 
part of the function, which rather elegantly allows us to re-use 
the previous technique of construction and a major part of 
the analysis. We also encourage the reader to read the 
corresponding section in \cite{RajabiWittGECCO20} for further 
insights into the construction.

In the following, we will imagine any
bit string $x$ of length~$n$ 
as being split into a prefix $a\coloneqq a(x)$ of length~$n-m$ and a suffix $b\coloneqq b(x)$ of 
length $m$, where $m$ is defined below.
 Hence, $x=a(x)\circ b(x)$, where $\circ$ denotes the concatenation.  
The prefix~$a(x)$ is called \emph{valid} if it is of the form 
$1^i 0^{n-m-i}$, \ie, $i$ leading ones and $n-m-i$ trailing zeros. The prefix fitness 
$\pre(x)$ of a string~$x\in\{0,1\}^n$ with valid prefix 
$a(x)=1^i0^{n-m-i}$ equals~$i$, the 
number of leading ones. The suffix consists of $\lceil \frac{1}{3}\sqrt{n}\rceil$ consecutive blocks of 
$\lceil n^{1/4}\rceil$ bits each, altogether $m\le \frac{1}{3}n^{3/4}=o(n)$ bits. 
Such a block is called \emph{valid} if it contains either~$0$ or $2$ one-bits; moreover, it is called \emph{active} if it contains~$2$ 
and \emph{inactive} if it contains~$0$ one-bits. A suffix where all blocks are valid and where all blocks following 
first inactive block are also inactive is called valid itself, and the suffix fitness $\suff(x)$ of a 
string~$x$ with valid suffix~$b(x)$  
is the number of leading active blocks before the first inactive one. Finally, we call $x\in\{0,1\}^n$ valid 
if both its prefix and suffix are valid.

The final fitness function is a weighted combination of $\pre(x)$ and $\suff(x)$. We define 
for $x\in\{0,1\}^n$, where $x=a\circ b$ with the above-introduced $a$ and~$b$, 
\begin{align*}
& \needglobalmut(x)\coloneqq \\
& \begin{cases}
 n^2 \suff(x) + \pre(x) & \text{\quad if $\pre(x)\le \frac{9(n-m)}{10}$ $\wedge$ $x$ valid}\\
 n^2 m + \pre(x) + \suff(x) - n - 1 \hspace{-2ex}& \text{\quad if $\pre(x)>\frac{9(n-m)}{10}$ $\wedge$ $x$ valid}\\
- \onemax(x) & 
  \text{\quad otherwise.} 
\end{cases}
\end{align*}

The function \needglobalmut equals $\needhighmut_\xi$ from 
\cite{RajabiWittGECCO20} for the setting $\xi=1/2$ (ignoring 
that $\xi<1$ was disallowed there for technical reasons). 
We note that all search points in the second case have a fitness of at least~$n^2 m - n - 1$, which 
is bigger than $n^2(m-1) + n$, an upper bound on the fitness of search points that fall into the first case 
without having $m$ leading active blocks in the suffix. Hence, search points~$x$ where $\pre(x)=n-m$ and 
$\suff(x)=\lceil \frac{1}{3}\sqrt{n}\rceil$ represent local optima of second-best overall fitness. The set of global optima 
equals the points where $\pre(x)=9(n-m)/10$ and $\suff(x)=\lceil \frac{1}{3}\sqrt{n}\rceil$, which implies that $(n-m)/10=\Omega(n)$ bits have to be flipped 
simultaneously to escape from the local toward the global optimum.

\begin{theorem}
\label{theo:needglobalmut}
With probability $1-o(1)$, \sdrls with $R\ge n$ needs $2^{\Omega(n)}$ 
steps to optimize \needglobalmut. The \oneoneea optimizes this 
function in time $O(n^2)$ with probability $1-2^{-\Omega(n^{1/3})}$.
\end{theorem}
\begin{proof}
As in the proof of Theorem~4.1 in \cite{RajabiWittGECCO20}, we have 
that the first valid search point (\ie, search point of non-negative fitness) of both \sdrls and \oneoneea 
has both  $\pre$- and $\suff$-value 
value of at most~$n^{1/3}$ with probability $2^{-\Omega(n^{1/3})}$. In the following, 
we tacitly assume that we have reached a valid search point of the described maximum $\pre$- and $\suff$-value and note that 
this changes the required number of improvements to reach local or global maximum only by a $1-o(1)$ factor. For readability 
this factor will not be spelt out any more.

As long as the counter threshold of \sdrls is not exceeded, 
the algorithm behaves like \rls. We the argumentation 
from Theorem~\ref{theo:unimodalrls} until the point in time 
where $\pre(x)=n-m$ since it is possible to improve the function 
value by one-bit flips before. Hence, 
the probability of ever increasing the $\suff$-value 
before $\pre(x)=n-m$  is at most $1/R\le 1/n$. The fitness 
can only be further improved if at least $(n-m)/10$ bits flip 
simultaneously. This requires the rate to be increased 
$(n-m)/10-1$ times; in particular the last of the increases 
happens only after a phase of length at least $\binom{n}{(n-m)/10-1} 
= 2^{\Omega(n)}$. This proves the statement for \sdrls.

We now analyze the success probability of the \oneoneea. To this 
end, we first bound 
the probability of a mutation being accepted after a valid search point has been reached. Even if a mutation changes up to $o(n)$ consecutive bits of the prefix 
or suffix, it must maintain $n-o(n)$ prefix bits in order to result in a valid search point. Hence, the probability of an accepted step 
at mutation probability $1/n$  is at most $(1-1/n)^{n-m-o(n)} = (1+o(1)) e^{-1}$.  Since the probability of flipping 
$\Omega(n)$ bits is $n^{-\Omega(n)}$, the probability 
of an accepted step 
is altogether, by the law of total probability,
$(1\pm o(1))(1-1/n)^{n} = (1\pm o(1)) e^{-1}$. By similar 
arguments, 
the probability of a mutation improving 
the $\pre$-value by $k$ at most $(1+ o(1)) e^{-1}/n^k$ 
and the probability 
of improving the $\suff$-value is at least 
$(1-o(1)) (e^{-1}/2) n^{-3/2}$ since there are $\binom{n^{1/4}}{2}=(1-o(1))n^{1/2}/2$ choices of probability at least $e^{-1}/n^2$ 
each. 

 We now consider a phase 
of $(3/4) e  mn$ steps. Using the bound on improving the $\suff$-value, we expect $(1-o(1))(3/8) \sqrt{n}$ 
activated blocks. By Chernoff bounds, with overwhelming probability 
we have at least $\frac{1}{3}\sqrt{n}$ such blocks. The probability of improving the $\pre$-value by~$k\ge 1$   
is only $(1+o(1))e^{-1} n^{-k}$, amounting to an 
expected number of 
  improvements by~$k$ of $(1+o(1))(3/4)  m n^{1-k} = 
	(1+o(1))(3/4)  n^{2-k}  $ 
and, using Chernoff bounds and union bounds over all 
$k=o(n)$, the probability of 
improving the \pre-value by at least 
$(9/10)m$ during the phase is $2^{-\Omega(n^{1/3})}$. 
 \end{proof}

\section{Minimum Spanning Trees}
\label{sec:mst}
Our self-adjusting $s$-flip mutation operator can also 
have advantages on classical combinatorial optimization problems. We 
reconsider the minimum spanning tree (MST) problem on which 
EAs and RLS were analyzed 
before \cite{NeumannW07}. The known bounds for the
globally searching \oneoneea are not tight. More precisely, 
they depend on $\log(w_{\max})$, the logarithm of 
the largest edge weight. This is different with 
RLS variants that flip only one or two bits due to an equivalence 
first formulated in \cite{RaidlKJ06}: if only up to two bits 
flip in each step, 
then the MST instance becomes indistinguishable from the MST 
instance formed by replacing all edge weights with their rank 
in their increasingly sorted sequence. This results in a tight upper 
bound of $O(m^2 \ln n)$, where $m$ is the 
number of edges, for RLS$^{1,2}$, an algorithm that uniformly 
at random decides to flip either one or two uniformly chosen bits
\citep{WittGECCO14}. Although not spelt out in the paper, 
it is easy to see that 
the leading term in the polynomial $O(m^2\ln m)$ 
is at most~$2$. This~$2$ stems from the logarithm of sum of the weight ranks,
which can be in the order of  $m^2$. We will see that the first 
factor of $2$ can, in some sense, be avoided in our \sdrlss.

The following theorem bounds the optimization time of \sdrlss in 
the case that the algorithm has reached a spanning tree and 
the fitness function only allows spanning trees to be accepted. It
is well known that with the fitness functions from \cite{NeumannW07},
the expected time to find the first spanning tree is $O(m\log m)$, 
which also transfers to \sdrlss; hence we do not consider this lower-order term 
further. However, our bound comes with an additional term related 
to the number of strict improvements. We will discuss this term after the proof.

\begin{theorem}
\label{theo:mst}
The expected optimization time of \sdrlss with $R=m^4$ on the MST problem with $m$ edges,  starting with an arbitrary spanning tree, is at most
\begin{align*}
& (1+o(1))\bigl((m^2/2)(1+\ln(r_1+\dots+r_m))+(4m\ln m)\expect{S}\bigr) \\
& \quad = (1+o(1)) \bigl(m^2\ln m 
+ (4m\ln m)\expect{S}\bigr),
\end{align*}
where $r_i$ is the rank of the $i$th edge in the sequence sorted 
by increasing edge weights and $\expect{S}$ is the expected number 
of strict improvements that the algorithm makes conditioned on that the strength
never exceeds~$2$.
\end{theorem}
\begin{proof}
We aim at using multiplicative drift analysis using 
$g(x)=\sum_{i=1}^m x_ir_i$ as potential function. Since the 
algorithm
has different states we do not have the same lower bound on 
the drift towards the optimum. However, at strength~$1$ no 
mutation is accepted since the fitness function 
from \cite{NeumannW07} gives a huge penalty to non-trees. Hence, 
our plan is to conduct the drift analysis conditioned 
on that the strength is at most~$2$ and account for the steps 
spent at strength~$1$ separately. Cases where the strength 
exceeds~$2$ will be handled by an error analysis and a restart argument.

Let $X^{(t)}\coloneqq g(x^{t})-g(\xopt)$ for the current 
search point~$x^{(t)}$ and an optimal search point $\xopt$. 
Since the algorithm behaves stochastically 
the same on the original fitness function~$f$ and the potential 
function~$g$, we obtain that $\expect{X^{(t)} - X^{(t+1)}\mid X^{(t)}} 
\ge X^{(t)}/\binom{m}{2}\ge 2X^{(t)}/m^2$ 
since the $g$-value can be 
decreased by altogether 
$g(x^{t})-g(\xopt)$ via a sequence of at most 
$\binom{m}{2}$ disjoint two-bit flips; see also the proof of 
Theorem~15 in \cite{DoerrJohannsenWinzenALGO12} for the underlying 
combinatorial argument. Let $T$ denote the number of steps at 
strength~$2$ 
until $g$ is minimized, assuming no larger strength to occur. 
Using the multiplicative drift theorem, 
we have $\expect{T}\le (m^2/2)(1+\ln(r_1+\dots+r_m)) \le 
(m^2/2)(1+\ln(m^2))$ and by 
the tail bounds for multiplicative drift (\eg, \citealp{LenglerDriftBookChapter}) 
it holds that 
$\prob{T> (m^2/2)(\ln(m^2) + \ln(m^2))} \le e^{-\ln(m^2)} = 1/m^2$. Note 
that this bound on $T$ is below the threshold for strength~$2$ 
since $\binom{m}{2}\ln R = (m^2-m) \ln (m^4) \ge 
(m^2/2)(4\ln m)$ for $m$ large enough. Hence, with probability 
at most~$1/m^2$ the algorithm fails to find the optimum before the 
strength can change from~$2$ to a different value 
due to the threshold being exceeded. 

We next bound the expected number of steps spent at larger strengths. 
Since each increase of the radius implies an unsuccessful phase 
at strength~$2$, the probability that radius $r$, where $3\le r\le m/2$, 
is selected before finding the optimum 
is at most $(1/m^2)^{r-2}$. 
According to Lemma~\ref{lem:partial-sum}, 
the number of steps spent for each such radius 
is at most $\frac{m-(r-1)}{m-(2r-1)}\binom{m}{r}$.  
By the law of total probability, the expected 
number of steps at larger 
strengths than~$2$ is at most
\[
\sum_{r=3}^{m/2} \frac{m-(r-1)}{m-(2r-1)}\binom{m}{r} \left(\frac{1}{m^2}\right)^{r-2} = o(m^2)
\]
and contributes only a lower-order term captured by the $o(1)$ 
in the statement of the theorem. If the strength exceeds~$2$, we wait 
for it become~$2$ again and restart the previous drift analysis, which 
is conditional on strength at most~$2$. Since the probability 
of a failure is at most~$1/m^2$, this accounts for an expected number 
of at most $1/(1-m^2)$ restarts, which is $1+o(1)$ as well.

It remains to bound the number of steps at strength~$1$. For each strict improvement,
the strength is reset to~$1$. Thereafter, $m\ln R=4m\ln n$ steps pass before
the strength becomes~$2$ again. Hence, if the strength does not exceed~$2$  
before the optimum is reached, this adds a term of $(4m\ln n)S$, where 
$S$ is the number of strict improvements in the run, to the running time. 
The expected number of strict improvements is bounded by $E(S)$, where we assume 
a random starting point of the algorithm and count the number of 
strict improvement after reaching the first tree. If an error occurs and the 
strength exceeds~$2$, the remaining expected number of strict improvements will 
not be bigger.
 \end{proof}

The term $\expect{S}$ appearing in the previous theorem is not easy to 
bound. If $\expect{S}=o(m)$,  the upper bound 
 suggests that \sdrls may be more efficient 
than the classical RLS$^{1,2}$ algorithm; with the caveat that we 
are talking about upper bounds only. However, it is not difficult 
to find examples where $\expect{S}=\Omega(m)$, \eg, on the worst-case graph 
used for the lower-bound proof in \cite{NeumannW07}, which we will study below 
experimentally, and we cannot generally rule out 
that $\expect{S}$ is asymptotically bigger than~$m$ on certain instances. 
However, empirically \sdrlss can be faster than RLS$^{1,2}$ and the \ooea 
on MST instances, as we will see in Section~\ref{sec:experiments}. 
In any case, although the 
algorithm can search globally, the bound in Theorem~\ref{theo:mst} 
does not suffer from the 
$\log(w_{\max})$ factor appearing in the analysis of the \oneoneea.

We also 
considered variants of \sdrlss that do not reset the strength to~$1$ after 
each strict improvement and would therefore, be able to work with 
strength~$2$ for a
long while on the MST problem. However, such an approach is risky in 
scenarios where, \eg, both one-bit flips and two-bit flips are possible 
and one-bit flips should be exploited for the sake of efficiency. Instead, we think that a combination 
of stagnation detection and selection hyperheuristics \citep{Warwicker19} based on the $s$-flip 
operator or the learning mechanism from \cite{DoerrDY16}, which performs very well on the MST,
would be more promising here.

\section{Experiments}
\label{sec:experiments}
In this section, we present the results of the experiments conducted to see the performance of the proposed algorithms for small problem dimensions. This experimental design was employed because our theoretical results are asymptotic.

\begin{figure} 
    \centering
	\includegraphics[width=\linewidth]{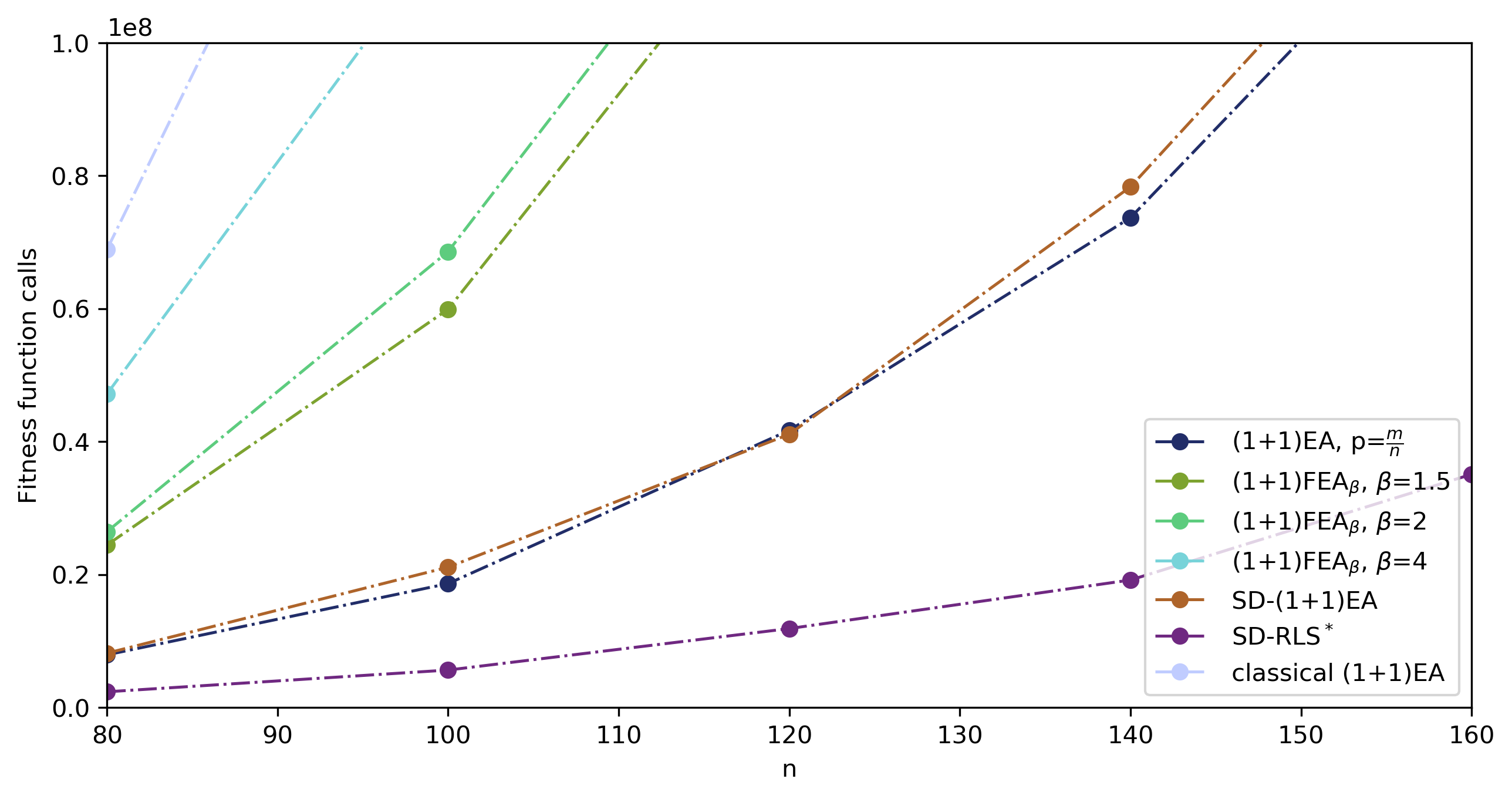}
	\caption{Average number of fitness calls (over 1000 runs) the mentioned algorithms took to optimize $\jump_4$.}\label{fig:exp_jump_line}
\end{figure}

	\begin{figure} 
		\includegraphics[width=\linewidth]{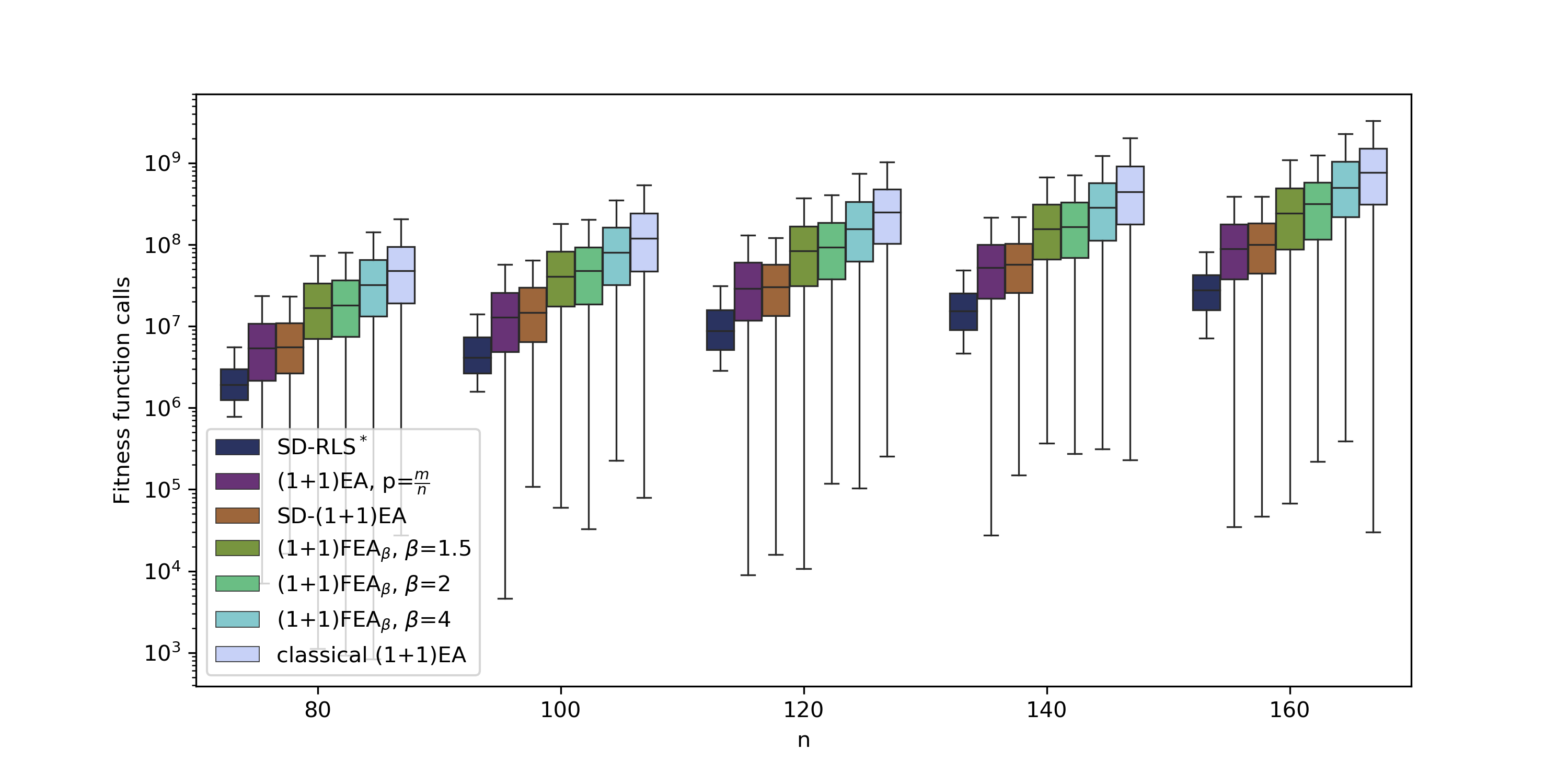}
		\caption{Box plots comparing number of fitness calls (over 1000 runs) the mentioned algorithms took to optimize $\jump_4$.} \label{fig:exp_jump_boxplot}
	\end{figure}
	
In the first experiment, we ran an implementation of Algorithm~\ref{alg:sdrls_star} (\sdrlss) on the \jump fitness function with jump size $m=4$ and $n$ varying from 80 to 160. We compared our algorithm against the \ooea with standard mutation rate $1/n$, the \ooea with mutation probability $m/n$, Algorithm \oofea from \cite{DoerrLMNGECCO17} with three different $\beta=\{1.5, 2, 4\}$, and the~\sdooea presented in \cite{RajabiWittGECCO20}. In Figure~\ref{fig:exp_jump_line}, we observe that \sdrlss outperforms the rest of the algorithms.

In the second experiment, we ran an implementation of four algorithms \sdrlss, \oofea with~$\beta=1.5$ from \cite{DoerrLMNGECCO17}, the standard \oneoneea and RLS$^{1,2}$ from \cite{NeumannW07} on the MST problem with the 
fitness function from \cite{NeumannW07} for two types of graphs called \TG and \erdosrenyi.

\begin{figure}
		\includegraphics[width=\linewidth]{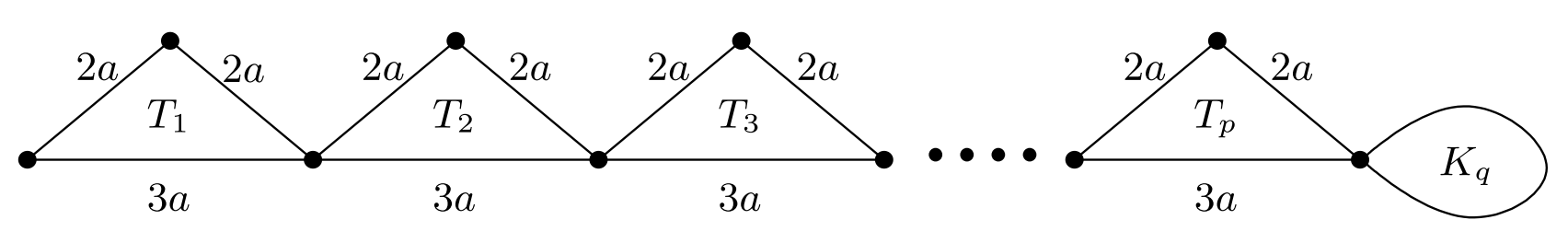}
		\caption{Example graph \TG with $p=n/4$ connected triangles and a complete graph on $q$ vertices with edges of weight 1. The source of the image is from \cite{NeumannW07}.} \label{fig:TG}
	\end{figure}

The graph \TG with $n$ vertices and $m=3n/4+\binom{n/2}{2}$ edges contains a sequence of $p=n/4$ triangles which are connected to each other, and the last triangle is connected to a complete graph of size $q=n/2$. Regarding the weights, the edges of the complete graph have the weight 1, and we set the weights of edges in triangle to $2a$ and $3a$ for the side edges and the main edge, respectively. In this paper, we consider $a=n^2$. The graph \TG is used for estimating lower bounds on the expected runtime of the \ooea and RLS in the literature \cite{NeumannW07}. In this experiment, we use $n=\{24,36,48,60\}$. As can be seen in Figure~\ref{fig:TGresults}, \oofea is faster than the rest of the algorithms, but \sdrlss outperforms the standard \oneoneea and RLS$^{1,2}$.

\begin{figure}
    \centering
    \subfloat[\centering Graphs \erdosrenyi]{{\includegraphics[width=5cm]{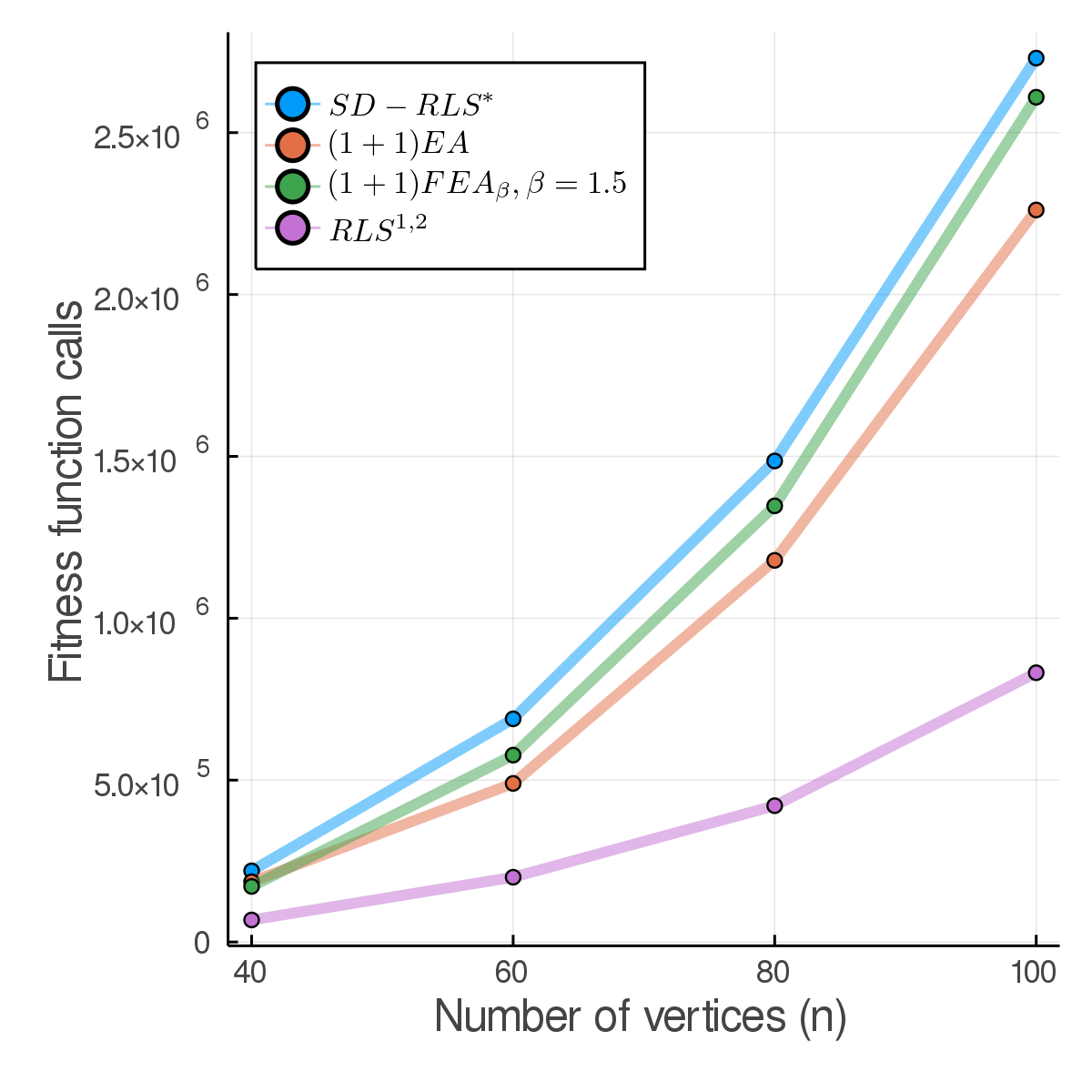} }\label{fig:ERresults}}
    \qquad
    \subfloat[\centering Graphs \TG]{{\includegraphics[width=5cm]{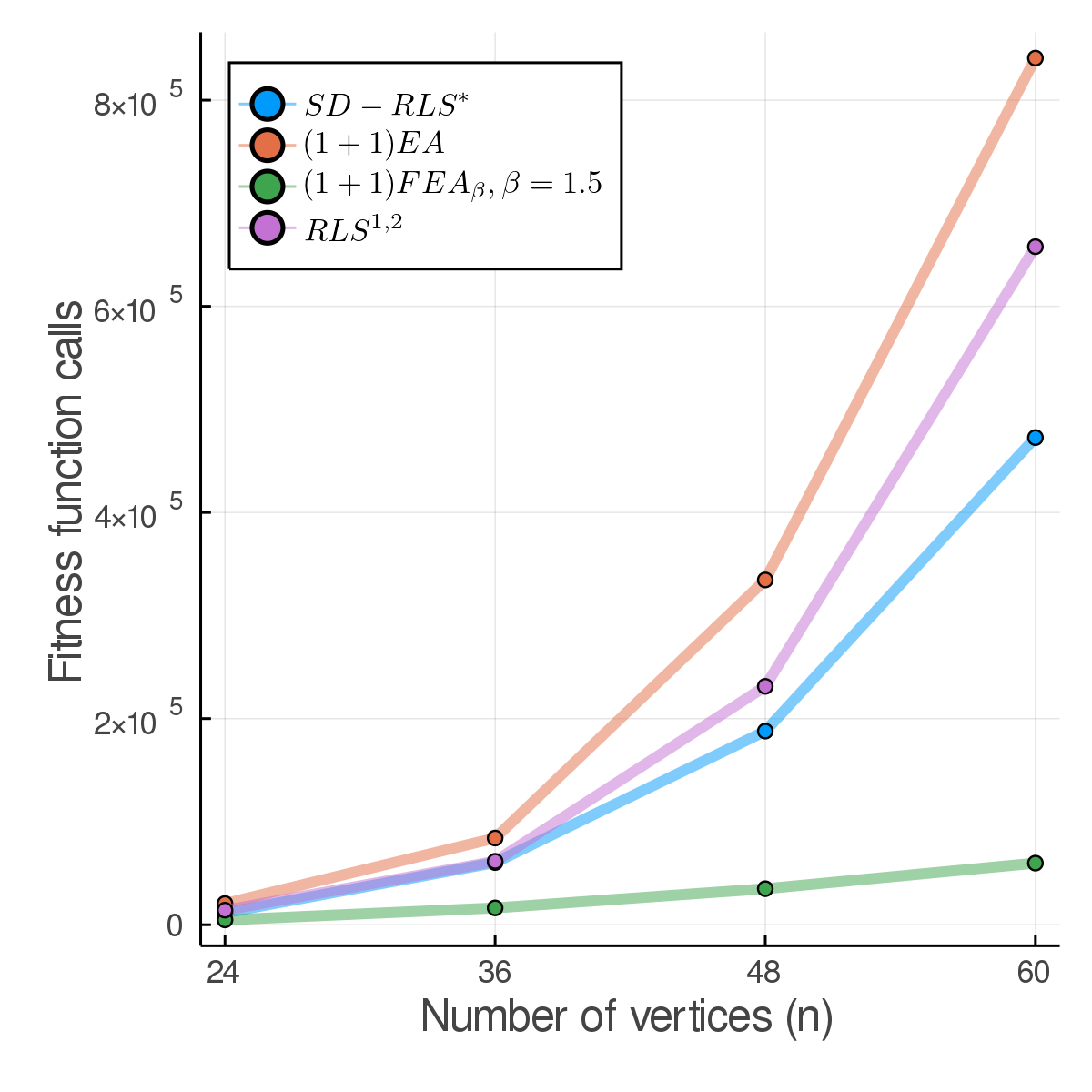} }\label{fig:TGresults}}
    \caption{Average number of fitness calls (over 400 runs) the mentioned algorithms took to optimize the fitness function MST of the graphs.}
    \label{fig:example}
\end{figure}

Regarding the graphs \erdosrenyi, we produced some random \erdosrenyi graphs with $p=(2\ln n)/n$ and assigned each edge an integer weight in the range $[1,n^2]$ uniformly at random. We also checked that the graphs certainly had a spanning tree. Then, we ran the implementation on MST of these graphs. The obtained results can be seen in Figure~\ref{fig:ERresults}. As we discussed in Section~\ref{sec:mst}, \sdrlss does not outperform the \ooea and RLS$^{1,2}$ on MST with graphs where the number of strict improvements in \sdrlss is large.

For statistical tests, we ran the algorithms on the graphs \TG and \erdosrenyi over 400 times, and all p-values obtained from a Mann-Whitney U-test between the algorithms, with respect to the null hypothesis of identical behavior, are less than $10^{-4}$ except for the results regarding the graph \TG with $n=24$.

\section*{Conclusions}
We have transferred stagnation detection, previously proposed for 
EAs with standard bit mutation, to the operator flipping exactly 
$s$ uniformly randomly chosen bits as typically encountered in randomized 
local search. Through both theoretical runtime analyses and experimental studies
we have shown that this combination of stagnation detection and local 
search efficiently leaves local optimal and often outperforms the previously 
considered variants with global mutation. We have also introduced techniques 
that make the algorithm robust if it, due to its randomized nature, misses 
the right number of bits flipped, and analyzed scenarios where global mutations 
 are  still preferable. In the future, we would like to investigate stagnation 
detection more thoroughly on instances of classical 
combinatorial optimization problem like the minimum spanning tree problem, for
which the present paper only gives preliminary but promising results.

\section*{Acknowledgement}
	This work was supported  by a grant by the Danish Council for Independent Research  (DFF-FNU  8021-00260B).

\bibliographystyle{mynatbib_english}

\bibliography{references}

\end{document}